\documentclass{article}
\usepackage[english]{babel}
\usepackage[utf8]{inputenc}
\usepackage{johd}
\usepackage{amsfonts}
\usepackage{graphics}
\usepackage{multirow}
\usepackage{amsmath}
\usepackage{amsthm}
\newtheorem{hypothesis}{Hypothesis}
\newtheorem{theorem}{Theorem}
\usepackage{caption}
\usepackage{subcaption}

\title{Proof of Quality: A Costless Paradigm for Trustless Generative AI Model Inference on Blockchains}

\author{Zhenjie Zhang$^{a*}$, Yuyang Rao$^{a}$, Hao Xiao$^{a}$, Xiaokui Xiao$^{b}, $Yin Yang$^{c}$  \\
        \small $^{a}$Nose Labs\\
        \small $^{b}$School of Computing, National University of Singapore, Singapore \\
        \small $^{c}$College of Science and Engineering, Hamad Bin Khalifa University, Qatar \\
        \small $^{*}$Corresponding author: \tt{zhenjie.zhang@nose.red} \\
}
\date{}

\begin{document}
\maketitle
\begin{abstract} 
Generative AI models, such as GPT-4 and Stable Diffusion, have demonstrated powerful and disruptive capabilities in natural language and image tasks. However, deploying these models in decentralized environments remains challenging. Unlike traditional centralized deployment, systematically guaranteeing the integrity of AI model services in fully decentralized environments, particularly on trustless blockchains, is both crucial and difficult. In this paper, we present a new inference paradigm called \emph{proof of quality} (PoQ) to enable the deployment of arbitrarily large generative models on blockchain architecture. Unlike traditional approaches based on validating inference procedures, such as ZKML or OPML, our PoQ paradigm focuses on the outcome quality of model inference. Using lightweight BERT-based cross-encoders as our underlying quality evaluation model, we design and implement PQML, the first practical protocol for real-world NLP generative model inference on blockchains, tailored for popular open-source models such as Llama 3 and Mixtral. Our analysis demonstrates that our protocol is robust against adversarial but rational participants in ecosystems, where lazy or dishonest behavior results in fewer benefits compared to well-behaving participants. The computational overhead of validating the quality evaluation is minimal, allowing quality validators to complete the quality check within a second, even using only a CPU. Preliminary simulation results show that PoQ consensus is generated in milliseconds, 1,000 times faster than any existing scheme.
\end{abstract}

\noindent\keywords{proof of quality, generative model, inference, blockchain, trustless, mechanism}\\

%\noindent\authorroles{For determining author roles, please use following taxonomy: \url{https://credit.niso.org/}. Please list the roles for each author.}\\

%\noindent{Data papers have a word limit of 1,000 words, excluding title, author affiliations, abstract, Tables/Figures and references, but including Tables/Figures captions and footnotes. To complete this data paper template, please delete or replace the text not given in bold with your own text.} 

\section{Introduction}

% the background of generative AI
%Generative AI has grown explosively in the last few years with powerful models generating images \citep{ho2020denoising}, texts \citep{brown2020language}, and videos \citep{yin2023survey} at high quality. These models have shown great potential in real-world use cases, such as auto coding, customer service chatbot, etc.

Deep learning has revolutionized the field of artificial intelligence, demonstrating significant advancements and high potential across various domains such as natural language processing (NLP), image recognition, and audio processing. In NLP, models like BERT and GPT-3 have set new benchmarks for tasks including machine translation, sentiment analysis, and question-answering, showcasing their ability to understand and generate human language with remarkable accuracy \citep{devlin2018bert, brown2020language}. Similarly, in image recognition, convolutional neural networks (CNNs) have achieved unprecedented performance, enabling applications ranging from autonomous driving to medical diagnostics \citep{krizhevsky2012imagenet, esteva2017dermatologist}. In the realm of audio processing, deep learning models have excelled in speech recognition and generation, as exemplified by systems like WaveNet and transformers-based architectures, which have significantly improved the quality and reliability of audio-based interactions \citep{van2016wavenet}. These breakthroughs highlight the transformative impact of deep learning on various aspects of AI, paving the way for future innovations and applications.

% generative AI on blockchains

%While a large number of open-source models have emerged as a competitive alternative to closed-source generative AI vendors, it remains challenging to deploy these models in completely decentralized environments, such as blockchains. The challenges mostly come from the nature of trustlessness in a decentralized ecosystem, such that the participant cannot easily trust the results generated by another participant. 

Combining AI and blockchain technologies is increasingly recognized as both important and necessary due to the complementary strengths and benefits each offers. AI excels in processing and analyzing large volumes of data, making intelligent predictions, and automating complex tasks, while blockchain provides a decentralized, secure, and transparent ledger system. Integrating AI with blockchain can enhance data security and integrity, ensuring that the data used by AI algorithms is tamper-proof and verifiable \citep{zhang2018fhirchain}. This synergy can mitigate the risks associated with data breaches and fraud, which are critical in sectors like finance, healthcare, and supply chain management. Additionally, blockchain can enhance the trustworthiness of AI systems by providing transparent audit trails, thus enabling more robust and accountable decision-making processes \citep{casino2019systematic}. Furthermore, decentralized AI models on the blockchain can democratize access to AI capabilities, preventing monopolistic control and fostering a more equitable distribution of technological benefits \citep{salah2019blockchain}. This integration is pivotal for driving innovation and ensuring that the deployment of AI systems aligns with principles of security, transparency, and fairness.

% technical requirements

Consensus is a widely used strategy in blockchain networks, where each computation is replicated across multiple nodes, and the outcome with the most agreements is accepted. However, this approach is unsuitable for generative AI models due to the substantial computational demands of model inference, which make replication across the network prohibitively expensive \citep{zheng2021agatha}. Additionally, the inherent high latency of consensus mechanisms would render the service impractical for real-world generative AI applications, as users require low-latency responses for effective functionality.

% limitations of existing solutions

To achieve low-latency, high-performance, and trustworthy generative model inference on blockchain, researchers are turning to advanced cryptographic techniques like zero-knowledge proofs (Ben-Sasson et al., 2014). However, these existing techniques face significant limitations in scalability and handling floating-point numbers, making them impractical for real-world generative AI models. For instance, ZKML \citep{weng2021mystique} aims to convert deep neural networks into circuits and generate proofs of inference operations, but it is only feasible for simple models with a few layers. Similarly, OPML \citep{conway2024opml} attempts to streamline the process based on optimistic assumptions, yet it requires hours to validate the inference of even small-scale Transformer-based models. Such low performance limitation makes it difficult to apply these strategies in real-world use cases of generative AI applications.

% Proof of Quality

In this paper, we propose a new paradigm called \emph{proof of quality} (PoQ). Drawing inspiration from the well-known \emph{proof of work} and \emph{proof of stake} strategies in classic blockchain architecture, PoQ focuses on validating the quality of model inference outputs rather than the inference process itself. The concept of PoQ is motivated by several observations. Firstly, in generative AI services, output quality is far more important than the inference process. Users are willing to pay for the service only when the generated response is satisfactory. Secondly, model output quality is not necessarily proportional to the computational workload of the model. While scaling laws generally indicate better performance with larger models \citep{kaplan2020scaling}, actual model behavior can vary, with some small yet well-optimized models outperforming larger ones. Therefore, it is more reasonable to allocate rewards to service vendors on the blockchain based on output quality rather than computational workload. Thirdly, validating model inference output is usually much easier than the inference process itself. State-of-the-art NLP models, such as Llama 3 \citep{touvron2023llama} and Mixtral \citep{jiang2023mistral}, typically contain billions of parameters. However, a slimmer model with millions of parameters can effectively assess the quality of outputs from these larger models.

PoQ is designed to support \textbf{trustworthy} model inference services on blockchains with \textbf{minimal overhead}. Specifically, given a user query, such as a question in plain text, the generative model is expected to generate a corresponding response, such as a text answer or an image. We divide the model reasoning process into three phases. In the inference phase, a single participant with strong computing power, such as GPUs, generates the output response by running a generative model with the query as input. In the assessment phase, the output is independently reviewed by a group of assessors in the network. The quality assessment is based solely on the pair of input and output. Finally, in the consensus phase, the network reaches an agreement on the overall quality score of the query-response pair, and rewards for all participants are determined accordingly.

%In the first \emph{inference} phase, one single participant with strong computing power, e.g., GPUs, is responsible for generating the output response by running a generative model with the query as input. In the second \emph{assessment} phase, the output is independently reviewed by a group of assessors in the network. The quality assessment is purely based on the pair of input and output. Finally, in the third \emph{consensus} phase, the whole network reaches an agreement on the overall quality score on the pair of query and response, and the rewards to all participants are decided accordingly.

%PoQ is a simple yet elegant approach to the trustless AI challenge on blockchains. The overhead is minimized because the computation overhead of the quality assessment, even when executed independently in the network, is lower than that of the generative model inference by orders of magnitude. This makes \emph{proof of quality} a preferable solution for deployment in any mainstream blockchain systems
PoQ offers a simple yet elegant solution to the trustless AI challenge on blockchains. The overhead is minimized because the computational burden of the quality assessment, even when performed independently across the network, is significantly lower than that of the generative model inference. This makes \emph{proof of quality} an ideal solution for deployment in mainstream blockchain systems.

% paper contribution summary

To summarize, this paper is expected to deliver the following research contributions based on the concept of \emph{proof of quality}:

\begin{enumerate}
    \item We provide a formal definition of the framework of \emph{proof of quality}.
    \item We design a concrete protocol, called \emph{Inference Quality Machine Learning} or PQML in short, based on \emph{proof of quality}, targeting NLP-based generative AI models.
    \item We theoretically analyze the behaviors of adversaries based on some assumptions to prove the robustness of our protocol.
    \item We discuss performance issues and propose implementation optimizations to enhance the efficiency and reliability of PQML protocol.
    \item We conduct experimental studies and present the results to validate our robustness and efficiency claims.
\end{enumerate}

The rest of the paper is organized as follows. Section \ref{sec:poq} introduces the basic concept of \emph{proof of quality}. Section \ref{sec:pqml} provides a concrete protocol based on PoQ targeting at NLP-based generative models. Section \ref{sec:analysis} analyzes the theoretical robustness of PQML against adversaries. Section \ref{sec:opt} depicts how to build PQML as a high-performance AI service. Section \ref{sec:exp} presents our empirical study results and finally Section \ref{sec:conclu} concludes the paper with future research direction discussion. 

\section{Proof of Quality}\label{sec:poq}
\begin{table}[t]
    \caption{Table of notations}
    \label{tab:notation_table}
    \vspace{-5pt}
\begin{center}
    \begin{tabular}{| c | l |}
       \hline
       Symbol & Notation \\
       \hline \hline
       $u$  & a user of the generative AI inference service \\ \hline
       $q\in\mathbb{Q}$  & Query or question from the generative AI system user \\ \hline
       $r\in\mathbb{R}$  & Response or answer to the user \\ \hline 
       $F$  & The generative AI model generating $r$ based on query $q$ \\ \hline
       $p$ & Zero-knowledge proof of the computation of $r=F(q)$ \\ \hline
       $s\in\mathbb{S}$ & Quality score on query-response pair $(q,r)$ \\ \hline
       $M$ & The quality score method, i.e., $s=M(q,r)$ \\ \hline
       $\mbox{Pu}_i$ & Public key for the $i$-th quality assessor \\ \hline
       $\mbox{Pr}_i$ & Private key for the $i$-th quality assessor \\ \hline
       $k$ & the number of quality scores for final quality consensus \\ \hline
       $\chi()$ & the inference reward calculation function \\ \hline
       $\alpha$ & scaling factor for inference reward function \\ \hline
       $\phi()$ & the quality assessment reward calculation function \\ \hline
       $\beta$ & scaling factor for quality assessment function \\ \hline
    \end{tabular}    
\end{center}
\end{table}

In this section, we formulate our \emph{Proof of Quality} framework, or PoQ in short. For better readability of the technical descriptions, all notations throughout the paper are summarized in Table \ref{tab:notation_table}. In PoQ framework, a user $u$ issues a query $q$ to the generative AI service. The generative AI model inference is formulated as a function $F:q\rightarrow r$, in which $q\in \mathbb{Q}$ is a query $q$ from the query domain $\mathbb{Q}$ and $r\in\mathbb{R}$ is a response from the response domain $\mathbb{R}$. The function $F$ is usually a computation-intensive model, such as GPT4 or stable diffusion.

\noindent \textbf{Standard Inference:} As a vanilla inference process, the engine runs the function $F$ straightforwardly and generates the output as $r=F(q)$. The problem with the standard inference is the lack of trustworthy guarantees. In a decentralized system, the user can hardly verify if the inference process is executed properly. A dishonest service provider could run a low-cost model for better profit.

\noindent \textbf{ZKML Inference:} generates two outputs, i.e., $(r,p)=F(q)$. Besides the response $r$, it also outputs a zero-knowledge proof $p$ of the computation process $r=F(q)$. Both $r$ and $p$ are sent to the user $u$. The user $u$ may want to validate the computation process based on the triple $(q,r,p)$ at his discretion.

\noindent \textbf{PQML Inference:} generates two outputs, i.e., $(r,v)=F(q)$. The verifiable result $v$ contains data and information helpful to computation validation. Given $v$, any validator could repeat the inference process in a VM to validate the process of $r=F(p)$. Different from ZKML inference, the user $u$ only receives the response $r$, while $(p,q,r)$ is sent to independent validators, who would challenge the result if any intermediate results in its validation do not match \citep{conway2024opml}.

\noindent \textbf{PoQ Inference:} generates the response $r$ only. The query-response tuple $(p,q)$ is sent to some quality assessors. The quality assessor runs a quality assessment method $M;\mathbb{Q}\times\mathbb{R}\rightarrow\mathbb{S}$, which maps every pair of query $q$ and response $r$ to a score $s$ in a fixed score domain $\mathbb{S}$. A typical score domain $\mathbb{s}$ is a closed set of real numbers, e.g., $[0,10]$, such that a higher number indicates a better quality of the response $r$. PoQ is the framework focusing on the quality instead of the computation process.

\begin{figure}[h]
\centering
     \includegraphics[width=1.0\textwidth]{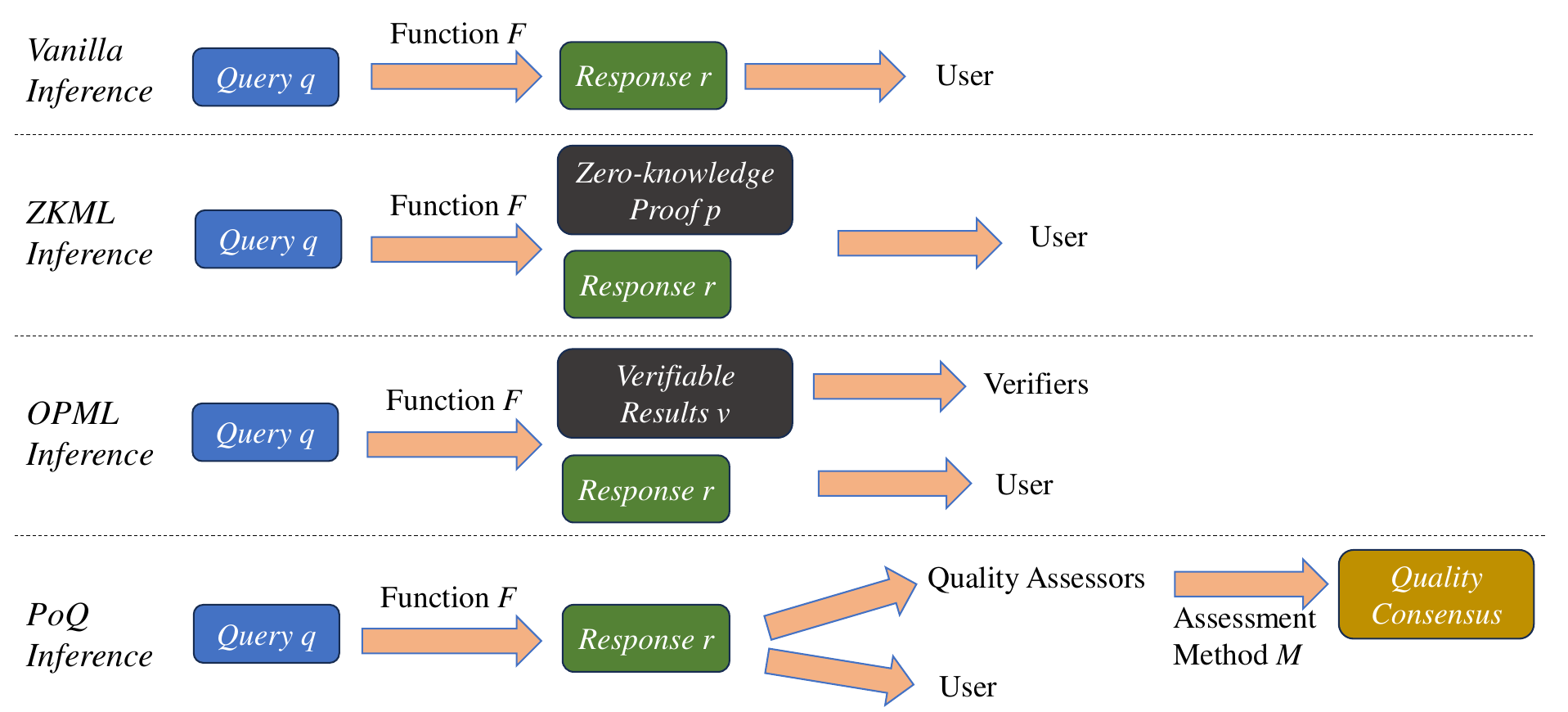}
     \vspace{-10pt}
      \caption{While standard inference only generates the response $r$ based on specified input query $q$, ZKML and PQML generate auxiliary information, i.e., zero-knowledge proof $p$ and verifiable results $v$ respectively, to enable the user or the verifiers to check the integreity of the whole computation process. PoQ adopts a very different strategy. It relies on third-party quality assessors to check the model output quality by evaluating the query-response pair $(q,r)$.}
       \label{fig:framework}
\end{figure}

In Figure \ref{fig:framework}, we summarize the workflows of all the inference frameworks above to illustrate their difference. Our PoQ framework is generic to handle all types of generative AI models. The concrete implementations of the framework, however, may vary because different output domains $\mathbb{R}$ require very different quality assessment methods. There is a huge difference between the quality of generated images and the quality of chatbot replies. It is thus crucial to select the suitable assessment method based on the domain $\mathbb{R}$. A good quality assessment method $M$ is expected to be

\begin{enumerate}
    \item Lightweight: the method's execution is fast, e.g., in milliseconds, and using limited computation resources, e.g., by CPUs.
    \item Robust: the method can distinguish good responses from bad ones and remain reliable when adversaries attempt to break the assessment process.
    \item Easy to implement: the method can be easily deployed in decentralized environments.
\end{enumerate}

The rest of the paper will provide technical evidence to prove all three requirements above could be satisfied for specific NLP-domain generative AI models. To fight against dishonest and lazy quality assessors, PoQ framework must also guarantee either of the following two conditions: 1) the quality assessment method $M$ could be fully validated, or 2) a group consensus on the score $s$ for any pair $(p,q)$ could be generated efficiently on the blockchain. The first condition could be satisfied by running ZKML or PQML over the quality assessment method $M$. The validation complexity is greatly reduced if $M$ is much less computation-intensive than the original model $F$. To the best of our knowledge, a proper quality assessment model remains unaffordable to any existing cryptographic solution. Instead, we attempt to propose an effective and efficient solution for the second condition, i.e., an approach to group consensus on quality scores. 

Depending on the complexity of the quality assessment method, there are different options on how the whole network generated consensus agreement on the quality score. If the computation overhead of the quality assessment method is insignificant, it is affordable to invite multiple participants to calculate the quality score independently. All scores from independent quality assessors are aggregated, by either averaging all scores or a simple majority voting. 

Finally, it is also important to design a reasonable incentive mechanism to avoid unexpected adversarial behavior in the system. The adversary may exploit a poorly designed incentive mechanism by generating low-quality output with minimal resources, yet take advantage of the quality assessment method $M$ to gain additional margins. We will discuss the incentive mechanism in more detail in Section \ref{sec:analysis}.

\section{PQML Protocol}\label{sec:pqml}

Based on the design space discussion of \emph{PoQ} framework in the last section, we design a concrete protocol, called \emph{Proof Quality based Machine Learning} or PQML in short, targeting specific generative AI models in the NLP domain. Specifically, the query domain $\mathbb{Q}$ and the response domain $\mathbb{R}$ are both plain texts in PQML. If the query $q$, for example, is \"What's the weather condition in next 2 hours?\", the response $r$ could be \"Mostly cloudy.\" Such domain setting fits almost all GPT-alike applications, such as service chatbot, auto coding, and others. In the rest of the section, we discuss our selections for concrete PQML in the design space of PoQ. The design decisions on these three aspects are closely connected, because the execution efficiency of quality assessment is crucial to the design of the consensus mechanism, and the incentive mechanism heavily depends on the consensus mechanism. 

\subsection{Quality Assessment Method}\label{sec:pqml:quality}

We are motivated by the huge success of the cross-encoder technique \citep{reimers2019sentence}, which is now commonly used as a filtering method for the Retrieval Augmented Generation (RAG) scheme \citep{lewis2020retrieval}. The basic idea is concatenating the text query $q$ and a text document. The combined string is fed into a Transformer-based model, e.g., BERT \citep{vaswani2017attention}, to generate a scalar score. The score measures the consistency between the query and response, indicating the goodness of the answer to the query. In PQML, we directly reuse the fully optimized cross-encoder models available on the Huggingface platform\footnote{\url{https://huggingface.co/cross-encoder}}. These models are optimized based on a large corpse containing pairs of questions and relevant documents during model training.

\begin{table}
\caption{We test different cross-encoder models on CPU and GPU and report performance on throughput and latency. The throughput is measured by the number of query-response pairs the system can process within a second. The latency is the average waiting time in seconds when processing a batch containing 1,000 query-response pairs because batch processing is way more efficient than processing one pair at a time. With an NVIDIA A100 GPU, we can easily process over 1,600 query-response pairs every second with latency within 0.7 seconds. The throughput is much lower when using CPUs instead of GPUs, but it still reaches 200+ throughput and 4 seconds latency in batch processing. }\label{table:cross-model-performance}
\begin{center}
\begin{tabular}{ | c | c | c | c| c | c |}
 \hline
 \multirow{2}{*}{Model} &\multirow{2}{*}{Size} & \multicolumn{2}{c|}{GPU(A100) } & \multicolumn{2}{c|}{CPU }  \\ \cline{3-6}
 & & Throughput & Latency (sec) & Throughput & Latency (sec) \\ \hline
 nli-deberta-v3-large & 1.7GB  & 1,666 & 0.61 & 48 & 20.5 \\ \hline
 nli-deberta-v3-medium & 0.7GB & 4,394 & 0.23 & 140 & 7.4 \\  \hline
 nli-deberta-v3-small &  0.5GB & 8,256 & 0.12 & 241 & 3.7 \\ \hline   
\end{tabular}
\end{center}
\end{table}

In Table \ref{table:cross-model-performance}, we summarize the performance statistics of mainstream cross-encoder models available in the open-source community. The results show that the inference latency and the operational cost of these models are much lower than common large language models (LLM), such as Llama 2 and Mixtral. All of them can be efficiently run on both CPUs and GPUs. Such low resource requirement makes the quality assessment step in PQML affordable.

\subsection{Consensus Mechanism}\label{sec:pqml:consensus}

The performance results in Table \ref{table:cross-model-performance} imply that cross-encoder models are very efficient, It is therefore affordable to deploy a simpler consensus mechanism with replicated computation to build a robust and fast-responding decentralized system. Each query-response pair is sent to $k>1$ quality assessors, and their quality scores are aggregated as the group consensus on the overall quality evaluation. The potential risk of the collective quality assessment method is the possibility of lazy quality assessors who may copy the scores from diligent quality assessors. The possible motivation of these lazy quality assessors is to reduce their computation costs and increase their profit.

To circumvent the potential problem, an encryption-based two-phase consensus mechanism is introduced to ensure every quality assessor must conduct an independent computation on the incoming pair of queries and responses. This new mechanism consists of two phases, as depicted in Figure \ref{fig:fast-consensus}. we need at least $k$ scores from different quality assessors. 

\noindent\textbf{Phase 1: }After receiving the query-response pair, the $i$-th quality assessor calculates the quality by running the specified cross-encoder model, i.e., $s_i = M(q,r)$. At the same time, the quality assessor generates a pair of public and private keys for encryption. Given the private key $\mbox{Pr}_i$ and public key $\mbox{Pu}_i$, the assessor encrypts the score by $\hat{s}_i = \mbox{Pr}_i(s)$ and writes $\hat{s}_i$ to an immutable file accessible to every participant in the network.

\noindent\textbf{Phase 2: }When sufficient quality assessment scores are collected, e.g., above a pre-defined threshold $k$, the quality assessors start to write their public keys $\{\mbox{Pu}_i\}$ to the shared file. Any participant could recover the scores and generate the average of original scores as the final assessment of the quality for response $s$, i.e., $\frac{1}{k}\sum_i \mbox{Pu}_i(\hat{s}_i)$.

The operational cost of the two-phase mechanism above is low because the public/private key generation and encryption process are both efficient on modern CPUs. It only incurs minimal latency on waiting for all quality assessors to finalize their quality scores. In Section \ref{sec:opt}, we will discuss additional performance optimization by skipping slow-responding quality assessors in the network.

\begin{figure}[t]
\centering
     \includegraphics[width=1.0\textwidth]{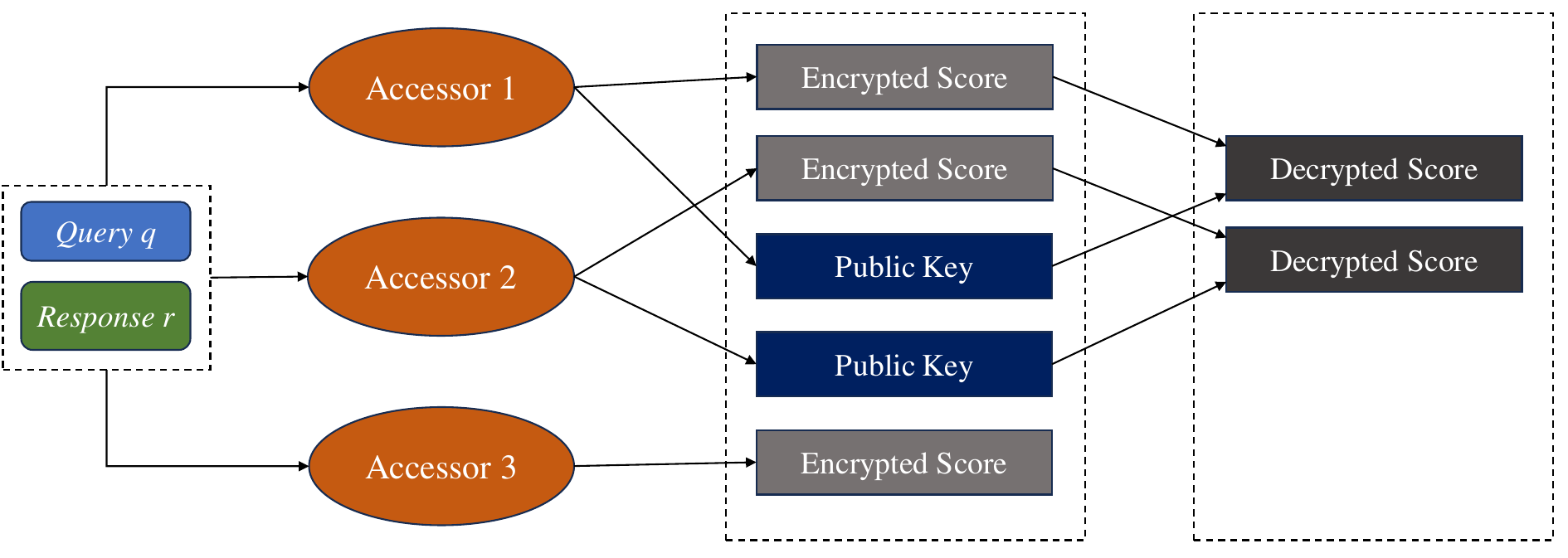}
     \vspace{-10pt}
      \caption{In this example, three assessors are responsible for generating the quality scores, and $k=2$ expected quality scores are needed for quality consensus. Assessor 1 and Assessor 2 publish their encrypted scores in a mutable and shared storage system. When $k=2$ scores are available, successful quality score nodes upload their corresponding public keys. A late arrival of an encrypted score from Assessor 3 is not included in the consensus calculation.}
       \label{fig:fast-consensus}
\end{figure}

\subsection{Incentive Scheme}\label{sec:pqml:incentive}

The incentive scheme in PQML is equally important to the consensus mechanism. The consensus mechanism alone does not encourage truth-telling behavior \citep{osborne2004introduction}. This is because the participants may not maximize their benefits by exactly following the protocol of PQML. The incentive scheme fills the gap, with a carefully designed reward system in which the participant's interests are maximized if they behave properly as an inference node or quality assessment node in the network.

\noindent\textbf{Inference Reward: }Given $k$ quality scores from assessors, i.e., $\{s_1,s_2,\ldots,s_k\}$ with each $s_i\in[L,U]$, the reward is a function $\chi(s_1,s_2,\ldots,s_k)$. Here $L$ and $U$ are the lower bound and upper bound of any quality score respectively. The output of the reward function $\chi$ is a real number between 0 and 1, which indicates the portion of reward the inference node could claim. There are two expected features of the reward function. First, the reward must be proportional to the quality score, so that the inference node receives more reward if its output response quality is higher. Second, the profit of operation with a more powerful model $F$ is higher than that of a less energy-consuming model $F'$, in the sense that the node owner would not make more earnings by replacing $F$ with a less powerful $F'$. Note that the less powerful model $F'$ may not refer to weaker models but also non-computation alternatives, such as returning random text strings or copying results from other vendors, e.g., ChatGPT.

Based on the expected features above, the inference reward function $\chi$ is designed as follows:

\begin{equation}
\chi(s_1,s_2,\ldots,s_k) = \exp\left(-\alpha\left(U-\frac{\sum_i s_i}{k}\right)\right),
\end{equation}

Here, $\alpha$ is the decaying scale factor, controlling how fast the reward decays with the descending overall quality score. The selection of $\alpha$ depends on the generative AI model market. We can always find a sufficiently large $\alpha$ to avoid any possibility of arbitrage by replacing better generative AI models with poorer ones, as later discussed in Section \ref{sec:analysis}. 

\noindent\textbf{Quality Assessment Reward: }We define $\phi(s_1,s_2,\ldots,s_k)$ as the reward distribution function, whose output is a vector of size $k$, i.e., $(h_1,h_2,\ldots,h_k)$, such that $\sum_i h_i = 1$ and $0\leq h_i\leq 1$ for each $i$. To enforce fairness among quality assessors, the reward distribution function $\phi$ is expected to 1) give more reward to assessors running the desired cross-encoder model; and 2) minimize the reward to assessors who do not properly evaluate the quality of query-response pair. This leads to the following design of function $\phi$, with $\bar{s}$ as the average of the scores:

\begin{equation}
h_i = \frac{\exp\left(-\beta\left(s_i-\bar{s}\right)^2\right)}{\sum_j \exp\left(-\beta\left(s_j-\bar{s}\right)^2\right)}
\end{equation}

A quality assessor claims more reward if its quality score is closer to the average score. If all quality assessors stick to the PQML protocol, due to the missing randomness in the cross-encoder model, the quality assessment reward is always equally assigned to each of the assessors, as

\begin{equation}
    \phi(s_1,s_2,\ldots,s_k)=\left(\frac{1}{k},\frac{1}{k},\ldots,\frac{1}{k}\right)
\end{equation}

The reward function above encourages the quality assessors to strictly use the specified quality assessment method $M$, since his own benefit would be hurt if it generates quality scores different from others. In the reward function in Formula (2), $\beta$ is a fixed parameter for all quality assessors. By using different $\beta_i$ for each quality assessor, the reward distribution is more biased to assessors with higher $\beta_i$, which could represent reputation or stake of the assessor in the network.

\section{Adversarial Analysis}\label{sec:analysis}

In this section, we focus on the analysis of the potential adversarial behavior of participants in the network, specifically the inference nodes and the quality assessor nodes. In the following, we first claim all hypotheses of the adversarial participants and then prove the robustness of PQML protocol against the adversaries.

\begin{hypothesis}\label{hypo:rational}
    All adversaries are \textbf{rational}, in the sense that they always attempt to maximize their profit when participating in the network.
\end{hypothesis}

The hypothesis above assumes that all participants in the network are profit-driven. They would not conduct anything \textit{improperly} if their behavior only results in less profit. It excludes potential adversaries who deliberately misbehave regardless of the economic loss.

To analyze the behavior of the inference nodes, we further assume there is a list of known and usable models in the market, i.e., $\{F_1,F_2,\ldots,F_n\}$. Given a query $q$, each model $F_j$ is associated with an expected quality score $e_j$ and operational cost $c_j$. For example, we could easily estimate the cost of calling APIs of GPT4 based on the input and expected response lengths. The following hypothesis regulates the possible distribution of $\bigcup_j\{(e_j,c_j)\}$.

\begin{hypothesis}\label{hypo:monotonic}
    There is \textbf{no domination} in the model market, such that for any pair of models $F_j$ and $F_l$, either $e_j>e_l$ or $c_j<c_l$.
\end{hypothesis}

Given two models $F_j$ and $F_l$, if $e_j\geq e_l$ and $c_j\leq c_l$, $F_j$ dominates $F_l$ because $F_j$'s expected quality is better than $F_l$'s and $F_j$'s cost is lower than $F_l$'s. $F_l$ is unlikely to survive in the market since no customer would ever choose $F_l$ as their product. Existing studies show that there always exists a stable market configuration in any domination-based market competition \citep{zhang2009domination}. Therefore, we could further claim all $e_j$ and $c_j$ for any model $F_j$ is static for a given period.

\begin{theorem}\label{the:inference}
Given a list of models $\{F_1,F_2,\ldots,F_n\}$ associated quality-cost matrix $\{(e_1,c_1),(e_2,c_2),\ldots,(e_n,c_n)\}$ in non-ascending order on $c_j$, there exists a threshold $\theta$, when $\alpha\geq\theta$ in the reward distribution function $\chi$, the inference node would always choose the most affordable model in the list.
\end{theorem}

\begin{proof}
To prove the theorem, we employ a constructive method to find the threshold $\theta$ based on the quality-cost matrix $\{(e_1,c_1),(e_2,c_2),\ldots,(e_n,c_n)\}$. To ensure any rational inference always prefers models of higher cost, we need to adjust the value of $\alpha$ to make a sufficiently large gap between the benefits of the model deployment options. For each pair of models $F_j$ and $F_l$, $j<l$, the gap of the benefit based on Formula (1) must be larger than the gap of their costs.

\begin{equation}
    \alpha_{jl} \geq \frac{\log (e_j - e_l)}{c_j - c_l}
\end{equation}

Therefore, it is safe to adopt the largest $\alpha_{jl}$ by iterating all pairs of $F_j$ and $F_l$, as

\begin{equation}
    \alpha \geq \max_{j,l}\alpha_{jl} =   \max_{j,l}\frac{\log (e_j - e_l)}{c_j - c_l}
\end{equation}

When $\alpha$ is no smaller than the value above, there is no motivation for inference nodes to deploy less accurate models from an economical perspective.
\end{proof}

The theorem above implies that PQML is robust against the adversary in our hypothesis by choosing the scaling factor $\alpha$. The proof itself also provides a simple method to estimate an appropriate setting of $\alpha$ based on the market setting, i.e., the availability of the models and their corresponding costs.

To understand the potential adversarial behavior of the quality assessors, we have some additional definitions and hypotheses. Consider a workload $\mathcal{W}$, which contains a \emph{countable} group of $m$ independent queries, i.e., $\mathcal{W}=\{q_1,q_2,\ldots,q_m\}$. In the NLP domains, for example, these queries comprise all possible text questions from all users. The distribution of quality score in terms of a model function $F$, workload $\mathcal{W}$ and a specific quality score method $M$ is then defined by all the output quality scores, i.e., $\mathcal{D}(\mathcal{W},F,M)$ is the collection of $\{M(q_1,F(q_1)),\ldots,M(q_m,F(q_m))\}$. 

\begin{hypothesis}\label{hypo:diversity}
The variance of the distribution of quality score associated with any model function $F$, real-world workload $\mathcal{W}$ and quality score method $M$ is no smaller than a constant $\Delta$, i.e.,
$$\mbox{Var}\left(\mathcal{D}(\mathcal{W},F,M)\right)\geq \Delta$$
\end{hypothesis}

The hypothesis above could be interpreted as a sufficiently large diversity of the output quality score. Even in the world-leading large language model, such as GPT4, the quality of the output may vary depending on the difficulty of the input query. It is therefore meaningful to assume such diversity is lower bounded by a constant $\Delta$. In Section \ref{sec:exp}, we will also present the statistics from real workload to consolidate the validity of the hypothesis. 

\begin{theorem}
Given the known lower bound $\Delta$ of all models, workloads, and quality score methods, if an adversary attempts to guess the quality score without running quality score method $M$, with a sufficiently small $\beta$, the expected reward of the adversary is smaller than any specified scalar $\epsilon$.
\end{theorem}

\begin{proof}
We again use the constructive method to identify the maximal $\beta$ based on the settings of $\Delta$ and $\epsilon$. When all nodes but $i$-th node stick to the quality assessment method $M$, and $i$-th node attempts to guess the score by randomly choosing a value from the distribution of the scores $\mathcal{D}(\mathcal{W},F,M)$. The expected reward of the $i$-th assessor is 
$$h_i=\frac{\exp(-\beta(s_i-\bar{s})^2)}{\exp(-\beta(s_i-\bar{s})^2) + k - 1}$$

By Jensen inequality \citep{boyd2004convex}, the expectation of the variable $h_i$ is no larger than

$$E(h_i)\leq \frac{\exp(-\beta \Delta)}{\exp(-\beta \Delta) + k - 1}$$

To make $E(h_i)$ smaller than $\epsilon$, it is sufficient when we have

$$\beta \leq \frac{1}{\Delta}\log \frac{k-1}{\epsilon^{-1}-1} $$

This shows that such $\beta$ always exists for any given $\Delta$ and $\epsilon$ and thus proves the theorem.
\end{proof}

\section{Performance Optimization}\label{sec:opt}

In previous sections, we propose the PQML framework, a concrete protocol for NLP applications, along with some adversarial analysis of the protocol. In practice, however, the efficacy of the network also depends on certain important architecture design decisions. In this section, we discuss a few optimizations to tackle some key challenges in the implementation of PQML.

\subsection{Fast Consensus}\label{sec:opt:wait}

%When multiple quality assessors are involved in the assessment step, the overall latency is decided by the slowest assessor. If any of the nodes fail to respond due to certain reasons, such as network failure, the system without \emph{early termination mechanism} will stop generating consensus on the response quality. However, if the waiting period is too short, the system may miss quality scores from decent participants. 

When multiple assessors are involved in the quality assessment step, the final latency of the consensus is determined by the response time of the slowest assessor. If any assessor node fails to respond, e.g., due to network failure, the system is no longer generating response quality consensus. To tackle the problem, we propose a revised mechanism to circumvent the potential risk.

If the system targets $k$ quality scores for the consensus, it assigns the quality score tasks to more than $k$ assessor nodes. Each assessor node attempts to finish the task as soon as possible. At the same time, they also monitor the progress of other nodes. When $k$ encrypted scores are published, they upload the public keys as described in Section \ref{sec:pqml}. The slow or non-responding nodes are excluded from the consensus, and therefore not entitled to claim the reward. Figure \ref{fig:fast-consensus} presents an example to illustrate the mechanism. In the example, only Assessor 1 and Assessor 2 are successful in quality score submission. 
Note that there is no incentive for the nodes to submit their public keys before at least $k$ encrypted scores are collected.

\subsection{Deterministic Node Selection}\label{sec:opt:select}

A subtle yet important question in PQML protocol is how to select the nodes in the network as inference and quality assessor nodes. In a decentralized blockchain-based system, it is crucial to balance fairness and efficiency. It would discourage participation if the system is unable to assign inference or quality assessment jobs to the nodes in the network. 

Another concern in the task node selection process is the overhead of randomness on the blockchain. Although it is technically feasible to generate verifiable random numbers \citep{gilad2017algorand} on blockchains, the cost of such a generation function is usually too high to afford. It is therefore unrealistic to adopt a traditional randomization strategy to select nodes for specific tasks. These observations motivate us to design a deterministic selection algorithm for both inference and quality assessment tasks.

Our deterministic task node selection algorithm consists of two parts, namely energy \emph{accumulation} and energy \emph{consumption}. Each node in the network is associated with two values, the energy value and the energy step. The energy value reflects how long the node has waited for a task, while the energy step indicates how fast the system would accumulate the energy value over time. A joining node is initialized by setting the energy value at 0 and the energy step at 1. There are two separate pools for inference node selection and quality assessment node selection. All nodes are ranked based on their current energy value in a non-ascending order. Ties are broken based on their join time. For inference node selection, the top-1 node is employed for query processing, while for quality assessment, the top-$k$ nodes are invited based on the parameter configuration on $k$. Once the nodes finish the corresponding task, their energy is reset to zero, giving opportunities to other nodes for the next task assignment.

\begin{figure}[t]
\centering
     \includegraphics[width=1.0\textwidth]{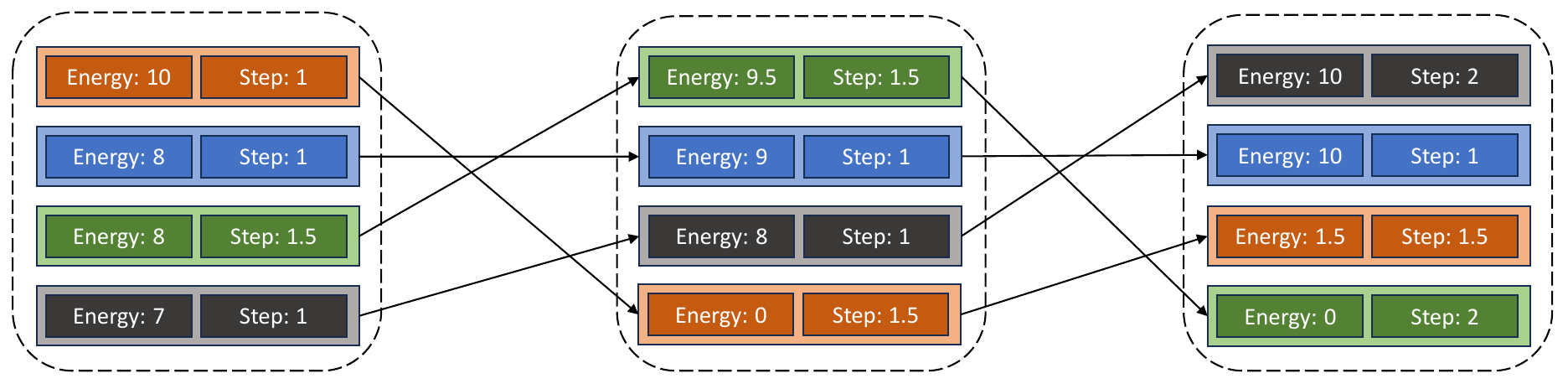}
     \vspace{-10pt}
      \caption{In this example, we illustrate a running example with 4 nodes in the inference pool. The orange node is selected for the first query processing because it has the highest energy value at 10. After the query is finished, the energy value of the orange node is reset to 0, while other nodes' energy values are incremented according to the step value. The green node overtakes the blue node because its step value is higher. The orange node's step value gains a 50\% bonus because of its good performance in the processing of the first query. After the completion of the second query, the grey node receives the additional bonus on the step value, because it has waited for too long for a concrete task. This grants the grey node a chance to process the third query in the running example.}
       \label{fig:fair-selection}
\end{figure}

There are two more configurations in the system, waiting threshold $T$ and the additional bonus $B$. If a node is not given any task in a period longer than $T$, it receives a complimentary bonus $B$ on its step value. This bonus is expected to grant the node more opportunities for task assignment. After the node is assigned to any task, this additional bonus $B$ is removed from the

If a node performs well in inference or quality assessment, the step value is updated with a bonus increase of 50\%. Otherwise, the initial step value remains as 1. This is used to encourage better performance for both inference and quality assessment tasks.

Instead of calculating the energy values and steps for each nodes periodically, it can be simplified to update only when a new query arrives. The selection node is responsible for re-calculating the energy values and steps for each node. In Figure \ref{fig:fair-selection}, an example of energy evolution is presented. 

%It illustrates the working mechanism. It encourages better performance by accelerating the performing nodes on rank improvement while ensuring each node always has a chance to be at the top position in the rank.

\section{Empirical Studies}\label{sec:exp}

The general goal of the empirical studies in this paper is to evaluate the effectiveness and efficiency of PQML in NLP-based generative AI applications. We therefore adopt a simulation-based strategy in the empirical studies, because both the effectiveness and efficiency of PQML could be fully verified by simulations.

\subsection{Effectiveness}\label{sec:exp:effective}

The effectiveness of our PoQ scheme depends on the accuracy of the scoring method on the output of large language model. In this part of the section, we focus on PQML and test the capability of cross-encoder on real workloads.

\noindent\textbf{Workload: }The workload in our experiment consists of three AI services (i.e., gpt-3.5-turbo-0125, gpt-4-turbo\footnote{\url{https://platform.openai.com/docs/models/gpt-4-turbo-and-gpt-4}}, Mistral-7b \footnote{\url{https://huggingface.co/mistralai/Mistral-7B-v0.1}}, Mixtral-8x7b \footnote{\url{https://huggingface.co/mistralai/Mixtral-8x7B-v0.1}} and Llama3-70b \footnote{\url{https://huggingface.co/meta-llama/Meta-Llama-3-70B}} as the generative models, one sentence transformer model\footnote{\url{https://huggingface.co/sentence-transformers/stsb-distilroberta-base-v2}} as cross-encoders and SQuAD \citep{rajpurkar2016squad} as the questions/queries used in our tests.

\noindent\textbf{Metrics: }The ground truth of the scores is generated by using GPT 4. For question and model/service, we feed its output to GPT4 for a quality evaluation between 0 and 10. Because the outputs of the cross-encoder are real numbers in $[-1,1]$, we \emph{normalize} the quality score from the cross-encoder to $[0,10]$ to make direct comparison possible. After normalization, we report the average score, its variance and the Pearson correlation between cross-encoder output and GPT ground truth.

\begin{table}[t]
\caption{We test with 100 questions from 100 different categories in SQuAD over 5 different generative AIs. We label the type of the AIs as a \emph{service} if we access the AI via public APIs of the vendor, or as a \emph{model} if it is an open-source model run locally.}\label{table:exp:effective}
\begin{center}
\begin{tabular}{ | l | l | c | c| c | }
 \hline
 Engine & Type & GPT Score (0-10) & Cross-Encoder Score (0-10) & Correlation \\ \hline
 GPT3.5 & Service & $9.04\pm 1.30$ & $7.88\pm 0.53$ & 0.12 \\ \hline
 GPT4 & Service & $9.24\pm 0.83$ & $7.89\pm 0.40$ & 0.13 \\ \hline 
 Mistral-7B & Model & $7.70\pm 2.15$ & $7.73\pm 0.50$ & 0.28 \\ \hline
 Mixtral-8x7b & Model & $6.97\pm 2.11$ & $7.59\pm 0.47$ & 0.35 \\ \hline
 Llama3-70b & Model & $7.33\pm 2.69$ & $7.81\pm 0.47$ & -0.06 \\ \hline
\end{tabular}
\end{center}
\end{table}

\noindent\textbf{Results: } All results are summarized in Table \ref{table:exp:effective}. There are some major observations on the results. First, the bias and variance of GPT score and Cross-Encoder score are different yet consistent. The average GPT score on GPT4 results is higher than the average GPT score of Mixtral-8x7b results by 2.1, while their difference on cross-encoder score is only 0.3. This is because of the internal bias to its own generated outputs when GPT4 is used to evaluate the correctness of other generative AIs. The cross-encoder score is more meaningful to human users. The variance of cross-encoder scores, on the other hand, is much smaller than the variance of GPT score. However, GPT4 and GPT 3.5 still significantly outperform the other models by a margin from 0.1 to 0.3. In Figure \ref{fig:reward}, we present the distribution of rewards based on Formula (1) in Section \ref{sec:pqml}, when different $\alpha$ is deployed. When $\alpha$ is small, PQML tends to give out rewards in a more friendly way, such that most of the responses could receive over 75\% of the reward. However, the difference between good and bad responses is also more significant. In Figure \ref{fig:reward:alpha0.5}, GPT4 obviously grabs more reward than Mixtral 8x7b. When using a larger $\alpha$, the reward difference diminishes, as shown in Figure \ref{fig:reward:alpha2}, and the advantage of GPT4 is much less significant.

\begin{figure}[h]
     \centering
     \begin{subfigure}[b]{0.33\textwidth}
         \centering
         \includegraphics[width=\textwidth]{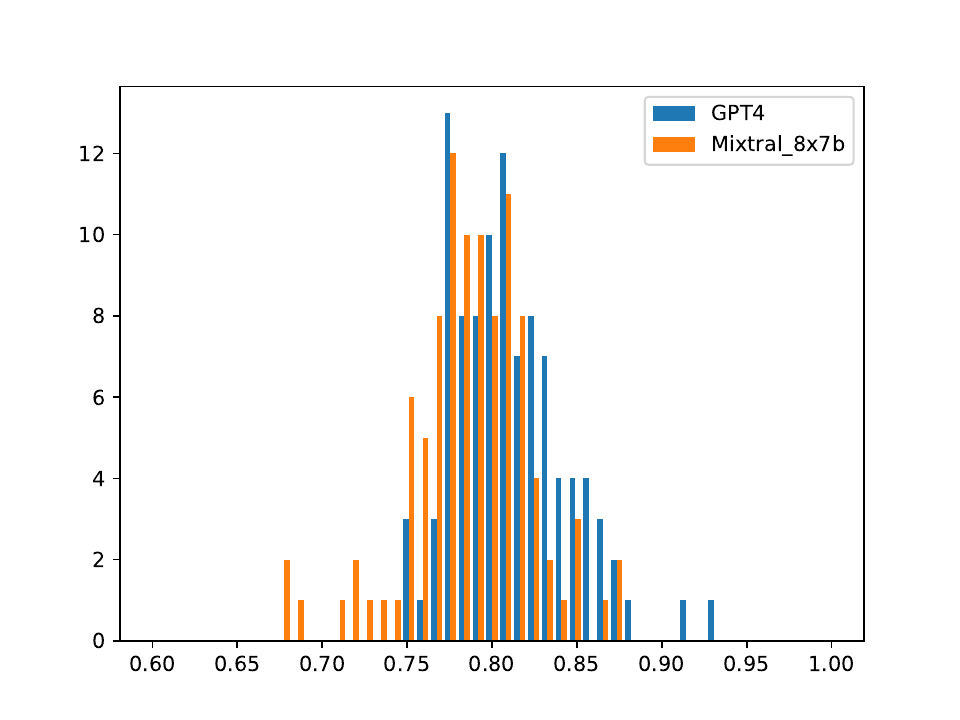}
         \caption{$\alpha=0.5$}
         \label{fig:reward:alpha0.5}
     \end{subfigure}
     \hspace{-10pt}
     \begin{subfigure}[b]{0.33\textwidth}
         \centering
         \includegraphics[width=\textwidth]{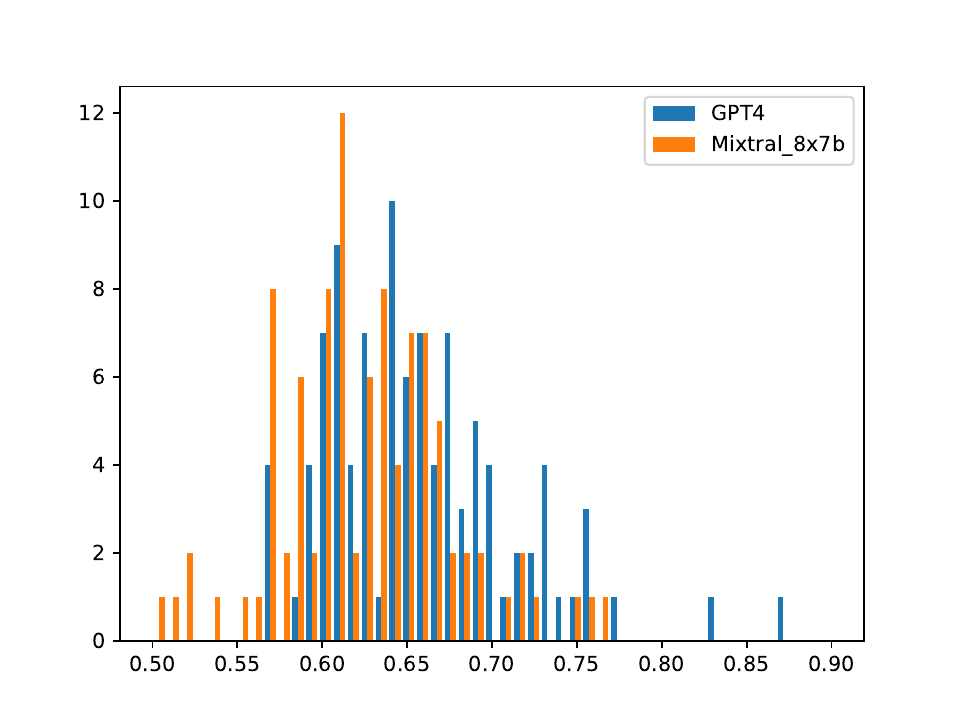}
         \caption{$\alpha=1$}
         \label{fig:reward:alpha1}
     \end{subfigure}
     \hspace{-10pt}
     \begin{subfigure}[b]{0.33\textwidth}
         \centering
         \includegraphics[width=\textwidth]{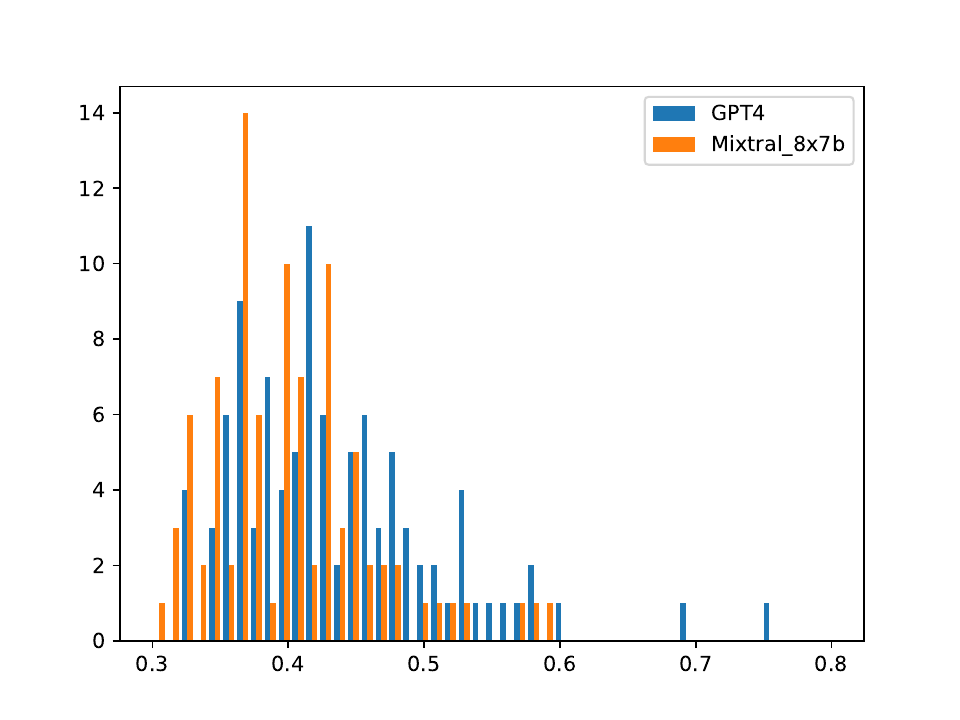}
         \caption{$\alpha=2$}
         \label{fig:reward:alpha2}
     \end{subfigure}
        \caption{Reward distribution of GPT4 and Mixtral 8x7b under different configurations of $\alpha$.}
        \label{fig:reward}
\end{figure}

Second, the correlation between the GPT score and the cross-encoder score is weak when the results are better, i.e., on GPT 3.5, GPT 4 and Llama3-70b. When the output quality is weaker, i.e., on Mistral-7b and Mixtral-8x7b, the correlation turns stronger. This implies that cross-encoder score is capable of capturing the missing of information in the generative AI output. It is unable to identify the subtle difference on the accuracy of the information itself.

\subsection{Efficiency}\label{sec:exp:efficient}

In this part of the section, we focus on the efficiency of PQML, specifically the speed of consensus, i.e., how fast the decentralized nodes could reach an agreement on the quality assessment and finally distribute the rewards to the quality assessors. This is important since most of the participants expect a quick settlement for each inference operation.

\noindent\textbf{Simulation Setup: }Instead of building a test network on blockchain, we simulate the consensus process by using Kafka as the distributed ledger system. There are three types of participating nodes in the simulation, \emph{management nodes}, \emph{inference nodes} and \emph{quality checking nodes}. The management node is responsible for handling query requests from the user, which chooses the inference node based on the deterministic priority scheme introduced in Section \ref{sec:opt:select}. After receiving a response from the inference node, the management node issues quality check requests to quality check nodes. All results are written to Kafka for recording. The management node also updates the energy and step size of each inference and management nodes, based on their performance in the query processing. The default setting of the simulation environment includes 1 single management node, 10 inference nodes and 30 quality checking nodes.

To investigate the impact of network size and controlling parameters, we test the end-to-end latency of the quality assessment phase when the number of assessors assigned to each inference task, i.e., $k$, varies from 1 to 30. When $k$ is a small number, the system relies on fewer quality assessors to evaluate the inference output. When $k$ grows, PQML consumes more computation resource from the quality assessors, but is expected to generate more reliable and stable quality assessments.

\noindent\textbf{Results: }In Figure \ref{fig:exp:latency}, we present the overall latency of consensus, i.e., the completion of quality validation of all $k$ validators, as well as the average latency of cross-encoder. The results show that the latency does not increase when PQML asks for more validator responses. This is because the inference with BERT-based cross-encoder is fast and stable, such that almost every validator could finish the quality assessment job within 35 seconds. Note that the consensus with Kafka is usually much faster than the consensus with blockchains. However, the additional overhead of blockchain is predictable and therefore does not much affect the efficiency of PQML consensus.

\begin{figure}[h]
\centering
     \includegraphics[width=.4\textwidth]{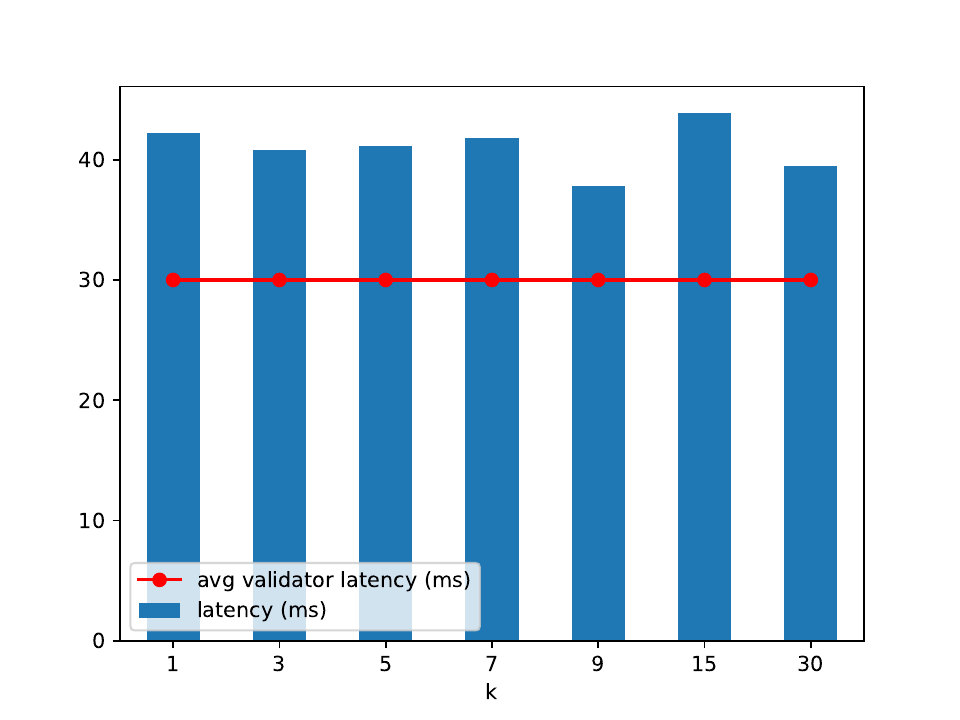}
     \vspace{-10pt}
      \caption{When varying the number of validators $k$ for inference tasks, the overall latency of consensus remains stable at around 50 milliseconds. We also plot the average latency of cross-encoder calculation at around 30 milliseconds.}
       \label{fig:exp:latency}
\end{figure}

It is unfortunately impossible to directly compare PQML against ZKML and OPML, because ZKML and OPML is usually only applicable to models of limited size. EZKL\footnote{\url{https://github.com/zkonduit/ezkl}} based on Halo2 is the most popular implementation of ZKML, which could generate zero-knowledge proofs for simple MLP neural network. OPML is known to be slower than ZKML when validating the computation in VMs, as stated by their authors \citep{conway2024opml}. It is therefore convincing that PQML is the only practical solution to NLP-based LLM inference on blockchain.

%\section{Related Work}\label{sec:related}
%In this section, we review the existing studies on AI computation on blockchain and inference validation techniques. Allora 

\section{Concluding Remarks}\label{sec:conclu}

In this paper, we present \emph{PoQ} as an alternative to classic zero-knowledge proof to enable off-chain generative model inference in a trustless environment. We demonstrate the effectiveness and efficiency of our approach for NLP-based LLM applications in both theoretical analysis and empirical studies. The inference evaluation consensus could be generated within a few seconds on blockchains.

In the future, it would be interesting to explore the following research directions. Firstly, in our current PQML setting, there is only one inference node selected to conduct the model reasoning job. When multiple inference nodes are engaged in the inference phase, a more complex reward mechanism is needed to ensure fair allocation of rewards based on their contributions. Second, the cross-encoder models used in our experiments are all pre-trained based on training data for document search. We can train a new BERT by using more question-answer pairs, to further improve the accuracy of the quality scores. Thirdly, PQML is only applicable to natural language applications. To cover image generation applications, it is necessary to design a proper quality assessment method for image outputs.

\bibliographystyle{johd}
\bibliography{bib}

%\section*{Supplementary Files (optional)}
%Any supplementary/additional files that should link to the main publication must be listed, with a corresponding number, title and option description. Supplementary files should also be cited in the main text.
%Note: supplementary files will not be typeset so they must be provided in their final form. They will be assigned a DOI and linked to from the publication.

\end{document}